%% file: cfl_eusipco24.tex
\documentclass[conference]{IEEEtran}
\IEEEoverridecommandlockouts
\usepackage{cite}
\usepackage{amsmath,amssymb,amsfonts}
\usepackage{amssymb,amsmath,amsthm, amsfonts}
\usepackage{algorithmic}
\usepackage{graphicx}
\usepackage{bm}
\usepackage{textcomp}
\usepackage{xcolor}
\usepackage{tikz} 
\usepackage{pgfkeys,pgfcalendar}
\usetikzlibrary{shapes}
\usetikzlibrary{shadows}
\usetikzlibrary{matrix,positioning}

\input{ml_macros.tex}

\def\BibTeX{{\rm B\kern-.05em{\sc i\kern-.025em b}\kern-.08em
    T\kern-.1667em\lower.7ex\hbox{E}\kern-.125emX}}
\begin{document}

\title{Analysis of Total Variation Minimization for Clustered Federated Learning\\
\thanks{This work has been supported by the Research Council of Finland under funding decision nr. $349966$ and $331197$.}
}
\newtheorem{theorem}{Theorem}
\newtheorem{assumption}{Assumption}

\tikzstyle{ncyan}=[circle, draw=cyan!70, thin, fill=white, scale=0.8, font=\fontsize{11}{0}\selectfont]
\tikzstyle{ngreen}=[circle,  draw=green!70, thin, fill=white, scale=0.8, font=\fontsize{11}{0}\selectfont]
\tikzstyle{nred}=[circle, draw=red!70, thin, fill=white, scale=0.8, font=\fontsize{11}{0}\selectfont]
\tikzstyle{ngray}=[circle, draw=gray!70, thin, fill=white, scale=0.55, font=\fontsize{14}{0}\selectfont]
\tikzstyle{nyellow}=[circle, draw=yellow!70, thin, fill=white, scale=0.55, font=\fontsize{14}{0}\selectfont]
\tikzstyle{norange}=[circle,  draw=orange!70, thin, fill=white, scale=0.55, font=\fontsize{10}{0}\selectfont]
\tikzstyle{npurple}=[circle,draw=purple!70, thin, fill=white, scale=0.55, font=\fontsize{10}{0}\selectfont]
\tikzstyle{nblue}=[circle, draw=blue!70, thin, fill=white, scale=0.55, font=\fontsize{10}{0}\selectfont]
\tikzstyle{nteal}=[circle,draw=teal!70, thin, fill=white, scale=0.55, font=\fontsize{10}{0}\selectfont]
\tikzstyle{nviolet}=[circle, draw=violet!70, thin, fill=white, scale=0.55, font=\fontsize{10}{0}\selectfont]
\tikzstyle{qgre}=[rectangle, draw, thin,fill=green!20, scale=0.8]

\author{\IEEEauthorblockN{Alexander Jung}
\IEEEauthorblockA{\textit{Department of Computer Science} \\
\textit{Aalto University}\\
Espoo, Finland \\
alex.jung@aalto.fi, ORCID}
}

\maketitle

\begin{abstract}
A key challenge in federated learning applications is the statistical heterogeneity of local datasets. 
Clustered federated learning addresses this challenge by identifying clusters of local datasets that 
are approximately homogeneous. One recent approach to clustered federated learning is generalized  
total variation minimization (GTVMin). This approach builds on a given similarity graph with weighted 
edges providing ``pairwise hints'' about the cluster assignments. While the literature offers a good 
selection of graph construction methods, little is know about the resulting clustering properties of 
GTVMin. We study conditions on the similarity graph to allow GTVMin to recover the inherent cluster 
structure of local datasets. In particular, under a widely applicable clustering assumption, we derive 
an upper bound for the deviation between GTVMin solutions and their cluster-wise averages. This 
bound provides valuable insights into the effectiveness and robustness of GTVMin in addressing 
statistical heterogeneity within federated learning environments.
\end{abstract}

\begin{IEEEkeywords}
machine learning, federated learning, distributed algorithms, convex optimization, complex networks
\end{IEEEkeywords}

\section{Introduction}
Federated Learning (FL) is an umbrella term for distributed optimization techniques to train 
machine learning (ML) models from decentralized collections of local datasets \cite{pmlr-v54-mcmahan17a,LiTalwalkar2020,Cheng2020,AgarwalcpSGD2018,Smith2017}. 
The most basic variant of FL trains a single global model in a distributed fashion from 
local datasets. However, some FL applications require to train separate (personalized) 
models for each local dataset \cite{Lengerich2018,LinUCB2010,Guk2019}. 

To train high-dimensional personalized models from (relatively) small local datasets, 
we can exploit the information provided by a similarity graph. The nodes of the similarity graph 
carry local datasets and corresponding local models. Undirected weighted edges in the similarity 
graph represent pairwise similarities between the statistical properties of local datasets. One natural approach 
to exploit the information provided by a similarity graph is generalized total variation minimization (GTVMin) \cite{ClusteredFLTVMinTSP}. 
GTVMin couples the training of personalized models via penalizing the variation of the model parameters 
across the edges of the similarity graph. 

We obtain different instances of GTVMin by using different measures of the variation of 
model parameters across the edges of the similarity graph. Two well-known special cases of GTVMin are 
``MOCHA'' \cite{Smith2017} and network Lasso \cite{NetworkLasso}. Our own recent work 
studies a GTVMin variant that can handle networks of different (including non-parametric) 
personalized models \cite{JuEusipco2023}. 

GTVMin is computationally attractive as it can be solved with scalable distributed optimization 
methods such as stochastic gradient descent or primal-dual methods \cite{Nedic2013,Rasch:2020tx}. 
Moreover, using a suitable choice for the similarity graph, GTVMin is able to capture the 
intrinsic cluster structure of local datasets \cite{ClusteredFLTVMinTSP}. 

{\bf Contribution.} We analyze the cluster structure of GTVMin instances that use the 
squared Euclidean norm to measure the variation of personalized model parameters. 
In particular, we provide an upper bound on the cluster-wise variability of model parameters 
learnt by GTVMin. This analysis complements our own recent work on the cluster structure 
of the solutions to GTVMin when using a norm to measure the variation of model parameters \cite{ClusteredFLTVMinTSP}. 

{\bf Outline.} Section \ref{sec_problem_formulation} formulates the problem of clustered 
FL (CFL) for distributed collections of data via generalized total variation minimization (GTVMin) 
over a similarity graph. Section \ref{sec_main_result} contains our main result which is an upper 
bound on the variation of learnt model parameters across nodes in the same cluster. 

\section{Problem Formulation}
\label{sec_problem_formulation} 

In what follows, we develop a precise mathematical formulation of clustered FL (CFL) 
over networks. Section \ref{sec_clustered_federated_learning} formulates the problem 
of learning personalized models for data generators that form clusters. Section \ref{sec_sim_graph} 
defines the concept of a similarity graph that provides information about the pairwise 
similarities between data generators. Section \ref{sec_gtvmin} then uses the similarity 
graph to formulate GTVMin. Our main result is an upper bound on the cluster-wise variability 
of local model parameters delivered by GTVMin (see Section \ref{sec_main_result}). 

\subsection{Clustered Federated Learning} 
\label{sec_clustered_federated_learning} 

We consider a collection of $\nrnodes$ data generators (or ``users'') that we index by $\nodeidx = 1,\ldots,\nrnodes$. 
Each data generator $\nodeidx$ delivers a local (or personal) dataset $\localdataset{\nodeidx}$. 
The goal is to train a personalized model $\localmodel{\nodeidx}$, with model parameters $\localparams{\nodeidx}$, 
for each $\nodeidx$. The usefulness of a specific choice $\localparams{\nodeidx}$ for the 
model parameters is measured by a non-negative local loss function $\locallossfunc{\nodeidx}{\localparams{\nodeidx}} \geq0$. 

The common idea of CFL methods is to pool (or cluster) local datasets with similar 
statistical properties. We can then train a personalized model using the pooled local 
datasets of the corresponding cluster. CFL is successful if the data generators 
actually form clusters, within which they are (approximately) homogeneous statistically. 
We make this requirement precise in the following clustering assumption. 
\begin{assumption} 
	\label{asspt_clustering}
Each data generator belongs to some cluster $\mathcal{C} \subseteq \{1,\ldots,\nrnodes\}$. 
There is a cluster-specific choice $\overline{\weights}^{(\mathcal{C})}$ for the local parameters 
for all $\nodeidx \in \mathcal{C}$ such that 
\begin{equation} 
	\label{equ_def_cluster_wise_opt}
	\sum_{\nodeidx \in \mathcal{C}} \locallossfunc{\nodeidx}{\overline{\weights}^{(\mathcal{C})}} \leq
	\varepsilon^{(\mathcal{C})}.
\end{equation} 
\end{assumption} 
Note that the Assumption \ref{asspt_clustering} might be valid for different choices of clusters.\footnote{In particular, there 
	might be two different clusters $\cluster_{1}, \cluster_{2}$ that both contain a specific node $\nodeidx \in \nodes$, each 
	satisfying Assumption  \ref{asspt_clustering} with (potentially) different parameters $	\varepsilon^{(\mathcal{C}_{1})},\varepsilon^{(\mathcal{C}_{2})}$.} 
Using larger clusters in Assumption \ref{asspt_clustering} requires a larger value $\varepsilon^{(\mathcal{C})}$ 
for \eqref{equ_def_cluster_wise_opt} to hold. Unless stated otherwise, we assume that the $\cluster$ consists 
of the nodes $\nodeidx=1,\ldots,|\cluster|$.  

{\bf Example.} It is instructive to consider Assumption \ref{asspt_clustering} for the special case 
of local linear regression. Here, generator $\nodeidx$ delivers $\localsamplesize{\nodeidx}$ data points 
$$ \pair{\featurevec^{(\nodeidx,1)}}{\truelabel^{(\nodeidx,1)}},\ldots,\pair{\featurevec^{(\localsamplesize{\nodeidx})}}{\truelabel^{(\localsamplesize{\nodeidx})}},$$
which we represent by the label vector $\labelvec^{(\nodeidx)}=\big( \truelabel^{(\nodeidx,1)},\ldots,\truelabel^{(\nodeidx,\localsamplesize{\nodeidx})} \big)^{T}$ 
and feature matrix $\mX^{(\nodeidx)} \defeq \big( \featurevec^{(\nodeidx,1)},\ldots,\featurevec^{(\nodeidx,\localsamplesize{\nodeidx})} \big)^{T}$. 
We assess local model parameters $\localparams{\nodeidx}$ via the local loss function $\locallossfunc{\nodeidx}{\localparams{\nodeidx}} = (1/\localsamplesize{\nodeidx}) \normgeneric{\vy^{(\nodeidx)} - \featuremtx^{(\nodeidx)} \localparams{\nodeidx} }{2}^{2}$. 
A sufficient condition for Assumption \ref{asspt_clustering} to hold with parameters $ \overline{\weights}^{(\mathcal{C})}, \varepsilon^{(\mathcal{C})}$ 
is that 
\begin{equation} 
\label{equ_def_probmodel_linreg_node_i}
\hspace*{-3mm}\labelvec^{(\nodeidx)}\!=\! \featuremtx^{(\nodeidx)}  \overline{\weights}^{(\mathcal{C})}\!+\!{\bm \varepsilon}^{(\nodeidx)},  \mbox{ for all } \nodeidx \in \mathcal{C}.
\end{equation} 
with noise terms ${\bm \varepsilon}^{(\nodeidx)}$ that are sufficiently small such that 
\begin{equation} 
\varepsilon^{(\mathcal{C})} \geq \sum_{\nodeidx \in \mathcal{C}} (1/\localsamplesize{\nodeidx}) 
\normgeneric{{\bm \varepsilon}^{(\nodeidx)}}{2}^{2}. \label{equ_clustering_error_local_lin_model}
\end{equation}


\subsection{Similarity Graph}
\label{sec_sim_graph} 

In general, we do not know to which cluster a given data generator $\nodeidx$ belongs to 
(see Assumption \ref{asspt_clustering}). However, we might still have some information 
about pair-wise similarities $\edgeweight_{\nodeidx,\nodeidx'}$ between any two 
data generators $\nodeidx,\nodeidx'$. We represent the pair-wise similarities between data 
generators by an undirected weighted ``similarity graph'' $\graph = \pair{\nodes}{\edges}$. 

The nodes $\nodes = \{1,\ldots,\nrnodes\}$ of this similarity graph $\graph$ are the data generators $\nodeidx=1,\ldots,\nrnodes$. 
An undirected edge $\{\nodeidx,\nodeidx'\}\!\in\!\edges$ between two different nodes (data generators) $\nodeidx,\nodeidx' \in \nodes$ indicates that they 
generate data with similar statistical properties. We quantify the extend of this similarity by a positive 
edge weight $\edgeweight_{\nodeidx,\nodeidx'}\!>\!0$. Figure \ref{fig_similarity_graph} depicts an 
example of a similarity graph that consists of three clusters. 
\begin{figure} 
	\begin{center}
		\begin{tikzpicture}[scale=11/5]
			\tikzstyle{every node}=[font=\small]
			\node[circle,fill=black] (C1_2) at (2,2.29) {};
			\node[left=1 cm of C1_2,circle,fill=black] (C1_1)  {};
			\node[below left =1cm and 1cm of C1_2,circle,fill=black] (C1_3)  {};
			\node[below =0.5cm of C1_2,circle,fill=black] (C1_4)  {};
			\node[circle,fill=black] (C3_3) at (4.3,2) {};
			\node[above left =0.4cm and 0.7cm of C3_3,circle,fill=black] (C3_2)  {};
			\node[below left =0.4 and 0.7cm of C3_3,circle,fill=black] (C3_4) {};
			\node[left =1.2cm of C3_3,circle,fill=black] (C3_1) {};
			\node[circle,fill=black] (C2_2) at (3.0,2.23) {};
			\node[below left =0.4cm and 0.4cm of C2_2,circle,fill=black] (C2_1)  {};
			\node[below right =0.4cm and 0.4cm of C2_2,circle,fill=black] (C2_3)  {};
			\node[above left = 0.2cm and 0.01cm of C1_2, font=\fontsize{8}{0}\selectfont,anchor=west] {$\localparams{\nodeidx}$}; 
			\node[below right = 0.5cm and 0.00cm of C2_1, font=\fontsize{8}{0}\selectfont,anchor=south] {$\localparams{\nodeidx'}$}; 
			
			\draw [line width=0.3mm,-] (C2_1)--(C1_2) node[draw=none,fill=none,font=\fontsize{8}{0}\selectfont,midway,above] {$\edgeweight_{\nodeidx,\nodeidx'}$};
			\draw [line width=0.6mm,-] (C1_2)--(C1_1);
			\draw [line width=0.4mm,-] (C1_2)--(C1_3);
			\draw [-] (C1_1)--(C1_3);
			\draw [-] (C1_3)--(C1_4);
			\draw [-] (C1_2)--(C1_4);
			\draw [line width=0.3mm,-] (C2_3)--(C3_1);
			\draw [line width=0.6mm,-] (C2_1)--(C2_2);
			\draw [line width=0.4mm,-] (C2_2)--(C2_3);
			\draw [line width=0.4mm,-] (C2_1)--(C2_3);
			\draw [-] (C3_1)--(C3_2);
			\draw [-] (C3_2)--(C3_3);
			\draw [line width=0.6mm,-] (C3_3)--(C3_4);
			\draw [-] (C3_2)--(C3_4);
			\draw [-] (C3_1)--(C3_4);
		\end{tikzpicture}
		\caption{\label{fig_similarity_graph} Example of a similarity graph whose nodes $\nodeidx \in \nodes$ 
			represent data generators and corresponding personalized models. Each personalized model is 
			parametrized by local model parameters $\localparams{\nodeidx}$. Two nodes $\nodeidx,\nodeidx'\in \nodes$ 
			are connected by an edge $\edge{\nodeidx}{\nodeidx'} \in \edges$ if the corresponding data generators 
			are statistically similar. The extend of similarity is quantified by a positive edge weight $\edgeweight_{\nodeidx,\nodeidx'}$ 
			(indicated by the thickness).}
	\end{center}
\end{figure} 

Ultimately, the similarity graph is a design choice for FL methods. This design choice might be guided by 
domain expertise: data generators being weather stations might be statistically similar if they are located 
nearby \cite{LocalizedLinReg2019}. Instead of domain expertise, we can also use established statistical 
tests to determine if two local datasets are obtained from a similar (identical) distribution \cite{Lee2023}. 

We can also obtain similarity measures for data generators via estimators for the divergence between 
probability distributions \cite{KLEst2018}. The edge weight $\edgeweight_{\nodeidx,\nodeidx'}$ 
can also be determined by a two-step procedure: (i) map each local dataset $\localdataset{\nodeidx}$ 
to a vector representation $\vz^{(\nodeidx)}$ and (ii) evaluate the Euclidean distance between the representations $\vz^{(\nodeidx)}$ and $\vz^{(\nodeidx')}$. 

Ideally, the connectivity of a similarity graph reflects the cluster structure of 
data generators: Nodes $\nodeidx \in \mathcal{C}$ in the same cluster (see 
Assumption \ref{asspt_clustering}) should be connected via many edges with large 
weight. On the other hand, there should only be few boundary (low-weight) 
edges that connect nodes in- and outside the cluster (see Figure \ref{fig_clustered_fl}). 

We measure the internal connectivity of a cluster via the second smallest 
eigenvalue $\eigval{2}(\mathcal{C})$ of the Laplacian matrix $\LapMat{\mathcal{C}}$ obtained 
for the induced sub-graph $\graph^{(\mathcal{C})}$.\footnote{The induced sub-graph 
	consists of the cluster nodes $\mathcal{C}$ and all edges $\edge{\nodeidx}{\nodeidx'} \in \edges$ of the 
	similarity graph $\graph$ with $\nodeidx, \nodeidx' \in \mathcal{C}$.}

The larger $\eigval{2}\big( \mathcal{C} \big)$, the better the connectivity among the 
nodes in $\mathcal{C}$. While $\eigval{2}\big( \mathcal{C} \big)$ describes the intrinsic 
connectivity of a cluster $\mathcal{C}$, we also need to characterize its connectivity 
with the other nodes in the similarity graph. To this end, we will use the cluster boundary 
\begin{equation}
	\label{equ_def_cluster_boundary}
	\bd[\mathcal{C}]\!\defeq\! \hspace*{-2mm}\sum_{\edge{\nodeidx}{\nodeidx'} \in  \partial \mathcal{C}} \hspace*{-4mm} \edgeweight_{\nodeidx,\nodeidx'} \mbox{, with } \partial \mathcal{C}\!\defeq\!\big\{ \edge{\nodeidx}{\nodeidx'}\!\in\!\edges: \nodeidx \in \mathcal{C}, \nodeidx'\!\notin\!\mathcal{C} \big\}. 
\end{equation}
For a single-node cluster $\mathcal{C} = \{ \nodeidx \}$, the cluster boundary 
coincides with the node degree, $\bd[\mathcal{C}]  = \sum_{\nodeidx' \neq \nodeidx} \edgeweight_{\nodeidx,\nodeidx'}$.

\begin{figure}[hbtp] 
	\begin{center} 
		\begin{tikzpicture}[auto,scale=0.6]
			
			\coordinate (i1) at (0,0);
			\coordinate (i2) at (-3,2);
			\coordinate (i3) at (-3,-2);
			
			\draw [dashed] (0.8cm,-0.2cm) arc [start angle=0, end angle=360, 
			x radius=3cm, 
			y radius=3.4cm]
			node [pos=0.5] {$\mathcal{C}$} 
			node [pos=.25] {} 
			node [pos=.5] {}  
			node [pos=.75] {};
			\coordinate (i4) at (3,0);
			\coordinate (i5) at (4,2);
			\draw [fill] (i1) circle [radius=0.2] node[below=5pt] {$\localparams{1}$};
			\draw [fill] (i2) circle [radius=0.2] node[below left = 5pt and 5pt of i2] {$\localparams{2}$};
			\draw [fill] (i3) circle [radius=0.2] node[below=5pt] {$\localparams{3}$};
			\foreach \nodeidx in {4,5}
			\draw [fill] (i\nodeidx) circle [radius=0.2] ; 
			\node[below right=2pt and 2pt of i4] {$\localparams{4}$};
			\node[above right=2pt and 2pt of i5] {$\localparams{5}$};
			\draw[line width=0.5mm] (i1) -- (i2);
			\draw[line width=0.5mm] (i2) -- (i3);
			\draw[line width=0.5mm] (i1) -- (i3);
			\draw[line width=0.5mm] (i1) -- node[midway,below]{$\partial \mathcal{C}$} (i4);
			\draw[line width=0.5mm] (i4) -- (i5);	
		\end{tikzpicture}
		\caption{	\label{fig_clustered_fl}  The similarity graph for a collection of data generators that include a cluster $\mathcal{C}$. 
			Ideally, a similarity graph contains many edges between nodes in $\mathcal{C}$ but only 
			few boundary edges (see \eqref{equ_def_cluster_boundary}) between nodes in- and outside $\mathcal{C}$.}
	\end{center} 

\end{figure}

\subsection{Generalized Total Variation Minimization} 
\label{sec_gtvmin} 

The goal of CFL is to train a local (or personalized) model $\localmodel{\nodeidx}$ for each 
data generator (or user) $\nodeidx$. Our focus is on local models that are parametrized by 
vectors $\localparams{\nodeidx} \in \mathbb{R}^{\dimlocalmodel}$, for $\nodeidx = 1,\ldots,\nrnodes$. 
The usefulness of a specific choice for the parameters $\localparams{\nodeidx}$ is measured by a 
local loss function $\locallossfunc{\nodeidx}{\localparams{\nodeidx}}$, for $\nodeidx = 1,\ldots,\nrnodes$.

In principle, we could learn $\localparams{\nodeidx}$ by minimizing $\locallossfunc{\nodeidx}{\localparams{\nodeidx}}$, i.e., 
implementing a separate empirical risk minimization for each node $\nodeidx \in \nodes$. 
However, this approach fails for a high-dimensional local model $\localmodel{\nodeidx}$ as 
they typically require much more training data than provided by the local dataset $\localdataset{\nodeidx}$. 

We can use the similarity graph to regularize the training of personalized models. In particular, we 
penalize local model parameters that result in a large total variation (TV): 
\begin{align} 
	\label{equ_def_tv}
 	\sum_{\edge{\nodeidx}{\nodeidx'} \in \edges} \edgeweight_{\nodeidx,\nodeidx'} \normgeneric{\weights^{(\nodeidx)}- \weights^{(\nodeidx')}}{2}^{2} & = \weights^{T} \big( \LapMat{\graph} \otimes \mI \big) \weights \nonumber \\ 
 	& \hspace*{-20mm} \mbox{ with } \weights \defeq \bigg( \big( \localparams{1}\big)^{T},\ldots, \big( \localparams{\nrnodes}\big)^{T} \bigg)^{T}.
\end{align} 
GTVMin balances between the average local loss (training errors) incurred by local model parameters and their 
TV \eqref{equ_def_tv}, 
\begin{align}
	\label{equ_def_gtvmin} 
	\big\{ \widehat{\weights}^{(\nodeidx)} \big\}_{\nodeidx=1}^{\nrnodes} & \in \argmin_{ \weights^{(\nodeidx)}} \bigg[ \sum_{\nodeidx \in \nodes} \locallossfunc{\nodeidx}{\weights^{(\nodeidx)}} + \nonumber \\ 
	& \hspace*{10mm}  \regparam \sum_{\edge{\nodeidx}{\nodeidx'} \in \edges} 
	\edgeweight_{\nodeidx,\nodeidx'}  \normgeneric{\weights^{(\nodeidx)}- \weights^{(\nodeidx')}}{2}^{2} \bigg] .
\end{align} 
The parameter $\regparam \geq 0$ in \eqref{equ_def_gtvmin} steers the preference for a small average local loss 
over a small TV. Choosing a large value for $\regparam$ results in solutions of \eqref{equ_def_gtvmin} to have 
a small TV (model parameters $\widehat{\weights}^{(\nodeidx)}$ vary little across 
edges $\edge{\nodeidx}{\nodeidx'} \in \edges$) even at the expense of a higher average local loss. 

If the similarity graph reflects the cluster structure of data generators (see Figure \ref{fig_clustered_fl}), 
GTVMin \eqref{equ_def_gtvmin} enforces the learnt parameter vectors $\big\{ \widehat{\weights}^{(\nodeidx)} \big\}_{\nodeidx=1}^{\nrnodes}$
to be approximately constant at cluster nodes (see Assumption \ref{asspt_clustering}). 
Note, however, that GTVMin \eqref{equ_def_gtvmin} does not require the knowledge about the 
actual clusters but only the similarity graph. 

We can interpret the (weighted edges of the) similarity graph as ``hints'' offered to GTVMin. 
If there are enough (and correct) hints, GTVMin recovers the actual cluster structure of data 
generators, i.e., the learnt model parameters \eqref{equ_def_gtvmin} are approximately identical 
for all nodes $\nodeidx$ in the same cluster $\cluster$. 

Our main result is an upper bound on deviation 
\begin{equation}
\label{equ_def_error_component_clustered}
\widetilde{\vw}^{(\nodeidx)} \defeq \estlocalparams{\nodeidx} - (1/|\mathcal{C}|)\sum_{\nodeidx' \in \mathcal{C}} \estlocalparams{\nodeidx'} \mbox{, for } \nodeidx \in \cluster,
\end{equation} 
between the learnt parameters $\widehat{\weights}^{(\nodeidx)}$ in the cluster $\cluster$ 
and their average. This upper bound will involve two key characteristics of a cluster 
$\cluster \subseteq \nodes$: the boundary \eqref{equ_def_cluster_boundary} and the second-smallest 
eigenvalue $\eigval{2}(\cluster)$ of the graph Laplacian $\LapMat{\cluster}$. This eigenvalue allows to lower bound the 
variation of local model parameters across $\mathcal{C}$, 
\begin{align}
	\label{equ_lower_bound_tv_eigval}
	\sum_{\substack{\nodeidx,\nodeidx' \in \mathcal{C} \\ \edge{\nodeidx}{\nodeidx'} \in \edges}} \hspace*{-2mm} \edgeweight_{\nodeidx,\nodeidx'}  \normgeneric{\weights^{(\nodeidx)}\!-\!\weights^{(\nodeidx')}}{2}^{2} & \!\geq\! \nonumber \\ 
	& \hspace*{-20mm} \eigval{2}(\mathcal{C})  \sum_{\nodeidx \in \cluster} \normgeneric{\weights^{(\nodeidx)}\!-\!{\rm avg}^{(\cluster)}\{\localparams{\nodeidx} \}}{2}^{2}.
\end{align}  
Here, $ {\rm avg}^{(\cluster)}\{\localparams{\nodeidx} \} \defeq (1/|\cluster|) \sum_{\nodeidx \in \cluster} \weights^{(\nodeidx)}$ 
is the average of the local model parameters of cluster nodes $\nodeidx \in \cluster$. The bound \eqref{equ_lower_bound_tv_eigval} 
can be verified via the Courant–Fischer–Weyl min-max characterization \cite[Thm. 8.1.2.]{GolubVanLoanBook} 
for the eigenvalues of the psd matrix $\LapMat{\mathcal{C}} \otimes \mI$. 

The RHS in \eqref{equ_lower_bound_tv_eigval} has a particular geometric interpretation: 
It is the squared Euclidean norm of the projection $\mathbf{P}_{\mathcal{S}^{\perp}} \weights^{(\cluster)}$ of 
the stacked model parameters ${\rm stack}^{(\cluster)} \big\{ \localparams{\nodeidx} \big\} \in \mathbb{R}^{\dimlocalmodel\cdot |\cluster|}$ 
on the orthogonal complement $\mathcal{S}^{\perp}$ of the subspace 
\begin{equation}
	\label{equ_def_subspace_S}
	\mathcal{S} \defeq  \bigg\{ \big( \vc^{T},\ldots,\vc^{T} \big)^{T} \mbox{ for some } \vc \in \mathbb{R}^{\dimlocalmodel} \bigg\} \subseteq \mathbb{R}^{\dimlocalmodel \cdot |\cluster|}. 
\end{equation} 

The subspace \eqref{equ_def_subspace_S} can also be used to decompose the 
estimation error $\Delta \weights^{(\nodeidx)} \defeq \estlocalparams{\nodeidx} - \overline{\weights}^{(\cluster)}$ 
of GTVMin \eqref{equ_def_gtvmin}. Indeed, by stacking the estimation error into a vector 
$\Delta \weights = {\rm stack}^{(\cluster)} \big\{ \Delta \weights^{(\nodeidx)}  \big\}$, 
we have the orthogonal decomposition
\begin{equation}
	\label{equ_def_orth_decomp_error}
\Delta \weights = 	\mP_{\mathcal{S}} \Delta \weights + 	\mP_{\mathcal{S}^{\perp}} \Delta \weights. 
\end{equation} 
We can evaluate the components in \eqref{equ_def_orth_decomp_error} as 
\begin{equation} 
	\label{equ_component_in_S}
	\mP_{\mathcal{S}} \Delta \weights  = {\rm stack}^{(\cluster)} \big\{ {\rm avg}^{(\cluster)}\{ \estlocalparams{\nodeidx} \} - \overline{\weights}^{(\cluster)} \big\}, 
\end{equation} 
and 
\begin{equation} 
		\label{equ_component_in_S_perp}
	\mP_{\mathcal{S}^{\perp}} \Delta \weights  = {\rm stack}^{(\cluster)} \big\{ \estlocalparams{\nodeidx} - {\rm avg}^{(\cluster)}\{ \estlocalparams{\nodeidx} \} \big\}.
\end{equation}

\section{Main Result}
\label{sec_main_result} 

Intuitively, we expect GTVMin \eqref{equ_def_gtvmin} to deliver (approximately) identical 
model parameters $\localparams{\nodeidx}$ for any cluster $\mathcal{C}$ that contains 
many internal edges but only few boundary edges. Using $\eigval{2}\big( \mL^{(\mathcal{C})} \big)$ 
as a measure for internal connectivity of $\cluster$ and the boundary measure $\bd[\mathcal{C}]$ (see \eqref{equ_def_cluster_boundary}) 
we can make this intuition precise. 
\begin{theorem}
\label{prop_upper_bound_clustered}
Consider a similarity graph $\graph$ whose nodes $\nodeidx \in \nodes$ represent data generators and 
correspond model parameters $\localparams{\nodeidx}$. 
We learn model parameters $\estlocalparams{\nodeidx}$, for each node $\nodeidx \in \nodes$, 
via solving GTVMin \eqref{equ_def_gtvmin}. If there is a cluster $\mathcal{C} \subseteq \nodes$ 
satisfying Assumption \ref{asspt_clustering},
\begin{align} 
\label{equ_upper_bound_err_component_cluster}
		\sum_{\nodeidx \in \mathcal{C}} \normgeneric{\widetilde{\vw}^{(\nodeidx)} }{2}^{2} & \!\leq\!\frac{1}{\regparam \eigval{2}\big( \mL^{(\mathcal{C})} \big)}  \bigg[\varepsilon^{(\mathcal{C})} 
 \!+\! \regparam \bd[\mathcal{C}] 2 \bigg( \normgeneric{\overline{\weights}^{(\mathcal{C})}}{2}^{2}\!+\!R^{2} \bigg)\bigg]. 
	\end{align} 
	Here, $R$ denotes an upper bound on the Euclidean norm $\normgeneric{\estlocalparams{\nodeidx}}{2}$ outside the cluster, i.e., 
	$\max_{\nodeidx \in \nodes \setminus \mathcal{C}} \normgeneric{\estlocalparams{\nodeidx}}{2} \leq R$. 
\end{theorem}
\begin{proof}
See Section \ref{sec_proof_main_result}. 
\end{proof} 
Note that Theorem \ref{prop_upper_bound_clustered} applies to any choice for the non-negative local loss 
functions $\locallossfunc{\nodeidx}{\cdot}$, for $\nodeidx=1,\ldots,\nrnodes$. In particular, the bound \eqref{equ_upper_bound_err_component_cluster} 
applies to any instance of GTVMin as long as the clustering Assumption \ref{asspt_clustering} holds. 

The usefulness of the upper bound \eqref{equ_upper_bound_err_component_cluster} depends on 
the availability of a tight bound $R$ on the norm of learnt model parameters outside the cluster $\cluster$. 
Such an upper bound can be found trivially, if the loss functions $\locallossfunc{\nodeidx}{\cdot}$ in \eqref{equ_def_gtvmin}
include an implicit constraint of the form $\normgeneric{\localparams{\nodeidx}}{2} \leq R$. 

We hasten to add that the bound \eqref{equ_upper_bound_err_component_cluster} only controls the 
deviation \eqref{equ_def_error_component_clustered} of the learnt model parameters $\estlocalparams{\nodeidx}$ 
from their cluster-wise average. This deviation coincides with the component \eqref{equ_component_in_S_perp} 
of the error $\estlocalparams{\nodeidx} - \overline{\weights}^{(\nodeidx)}$. The bound \eqref{equ_upper_bound_err_component_cluster} 
does not tell us anything about the other error component \eqref{equ_component_in_S}. 

Theorem \ref{prop_upper_bound_clustered} covers single-model FL \cite{pmlr-v54-mcmahan17a,SmithCoCoA} as the 
extreme case where all nodes belong to a single cluster $\mathcal{C} = \nodes$. 
Trivially, the cluster boundary then vanishes and the bound \eqref{equ_upper_bound_err_component_cluster} 
specializes to 
\begin{equation} 
	\nonumber 
\sum_{\nodeidx \in \mathcal{C}} \normgeneric{\widetilde{\vw}^{(\nodeidx)} }{2}^{2} \!\leq\!\frac{\varepsilon^{(\mathcal{C})} }{\regparam \eigval{2}\big( \mL^{(\mathcal{C})} \big)}. 
\end{equation} 
Thus, for the single-model setting (where $\cluster = \nodes$) the error component \eqref{equ_def_error_component_clustered} 
can be made arbitrarily small by choosing the GTVMin parameter $\regparam$ sufficiently large. 



\section{Proof of Theorem \ref{prop_upper_bound_clustered}} 
\label{sec_proof_main_result} 

We verify \eqref{equ_upper_bound_err_component_cluster} via a proof by contradiction, i.e., 
we show that if \eqref{equ_upper_bound_err_component_cluster} would not hold, then 
$\estlocalparams{\nodeidx}$ cannot be a solution to \eqref{equ_def_gtvmin}. To this end, 
we decompose the objective function in GTVMin \eqref{equ_def_gtvmin} as follows: 
\begin{align}
	\label{equ_def_f_prime_two_prime}
	f(\weights) & = \nonumber \\ 
	& \hspace*{-10mm} \underbrace{\sum_{\nodeidx \in \mathcal{C}} \locallossfunc{\nodeidx}{\localparams{\nodeidx}}\!+\!\regparam  \sum_{\substack{\edge{\nodeidx}{\nodeidx'} \in \edges\\ \nodeidx \in \mathcal{C}}} 
		\edgeweight_{\nodeidx,\nodeidx'}  \normgeneric{\weights^{(\nodeidx)} \!-\!\weights^{(\nodeidx')}}{2}^{2}}_{=: f'\big( \weights \big)} \nonumber \\ 
	& + f''\big( \weights \big). 
\end{align} 
Here, we used the stacked local model parameter $\weights = {\rm stack} \big\{ \localparams{\nodeidx} \big\}_{\nodeidx=1}^{\nrnodes} \in \mathbb{R}^{\dimlocalmodel \cdot \nrnodes}$. Note that only the first component $f'$ in \eqref{equ_def_f_prime_two_prime} depends on the 
local model parameters $\localparams{\nodeidx}$ at the cluster nodes $\nodeidx \in \cluster$.

Let us introduce the shorthand $f'\big( \weights^{(\nodeidx)} \big)$ for the function obtained from $f'(\weights)$ 
for varying $\localparams{\nodeidx}$, $\nodeidx \in \mathcal{C}$, but fixing $\localparams{\nodeidx} \defeq \estlocalparams{\nodeidx}$ 
for $\nodeidx \notin \mathcal{C}$. We verify the bound \eqref{equ_upper_bound_err_component_cluster} 
by showing that if it does not hold, the local model parameters $\overline{\weights}^{(\nodeidx)} \defeq \overline{\weights}^{(\mathcal{C})}$, 
for $\nodeidx \in \mathcal{C}$, results in a smaller value $f'\big( \overline{\weights}^{(\nodeidx)} \big) < f'\big( \widehat{\weights}^{(\nodeidx)} \big)$ 
than the choice $\estlocalparams{\nodeidx}$, for $\nodeidx \in \mathcal{C}$. This would contradict the 
fact that $\widehat{\weights}^{(\nodeidx)}$ is a solution to \eqref{equ_def_gtvmin}. 

Then, note that 
\begin{align}
	f'\big( \overline{\weights}^{(\nodeidx)} \big)  &= 	 \sum_{\nodeidx \in \mathcal{C}} \locallossfunc{\nodeidx}{\overline{\weights}^{(\nodeidx)}}  \nonumber \\ 
	& \hspace*{-10mm}+\hspace*{-2mm} \sum_{\substack{\edge{\nodeidx}{\nodeidx'} \in \edges \\ \nodeidx, \nodeidx' \in \mathcal{C}}} 	\hspace*{-2mm}
 \regparam	\edgeweight_{\nodeidx,\nodeidx'}  \normgeneric{\overline{\weights}^{(\mathcal{C})} \!-\!\overline{\weights}^{(\mathcal{C})}}{2}^{2}+ \hspace*{-2mm} \sum_{\substack{\edge{\nodeidx}{\nodeidx'} \in \edges \\ \nodeidx \in \mathcal{C}, \nodeidx' \notin \mathcal{C}}}
	\hspace*{-2mm}	\regparam\edgeweight_{\nodeidx,\nodeidx'}  \normgeneric{\overline{\weights}^{(\mathcal{C})} \!-\!\widehat{\weights}^{(\nodeidx')}}{2}^{2}  \nonumber \\
	& \stackrel{\eqref{equ_def_cluster_wise_opt}}{\leq}
	\varepsilon^{(\mathcal{C})} + \regparam \sum_{\substack{\edge{\nodeidx}{\nodeidx'} \in \edges \\ \nodeidx \in \mathcal{C}, \nodeidx' \notin \mathcal{C}}}
	\edgeweight_{\nodeidx,\nodeidx'}  \normgeneric{\overline{\weights}^{(\mathcal{C})} \!-\!\widehat{\weights}^{(\nodeidx')}}{2}^{2} \nonumber \\
	&\stackrel{(a)}{\leq} \varepsilon^{(\mathcal{C})} 
	+ \regparam \sum_{\substack{\edge{\nodeidx}{\nodeidx'} \in \edges \\ \nodeidx \in \mathcal{C}, \nodeidx' \notin \mathcal{C}}}
	\edgeweight_{\nodeidx,\nodeidx'} 2 \bigg( \normgeneric{\overline{\weights}^{(\mathcal{C})}}{2}^{2}+ \normgeneric{\widehat{\weights}^{(\nodeidx')}}{2}^{2} \bigg)  \nonumber \\ 
	& \leq  \varepsilon^{(\mathcal{C})}
	+ \regparam \bd[\mathcal{C}] 2 \bigg( \normgeneric{\overline{\weights}^{(\mathcal{C})}}{2}^{2}\!+\!R^{2} \bigg).  \label{equ_norm_error_func_value_clustered}
\end{align} 
Step $(a)$ uses the inequality $\normgeneric{\vu\!+\!\vv}{2}^{2} \leq 2\big(\normgeneric{\vu}{2}^{2}\!+\!\normgeneric{\vv}{2}^{2}\big)$ which 
is valid for any two vectors $\vu,\vv \in \mathbb{R}^{\dimlocalmodel}$. 

On the other hand, 
\begin{align}
	\label{equ_obj_function_gtvmin_average_ac_component_clustered}
	f' \big( \widehat{\weights}^{(\nodeidx)} \big) & \geq \regparam \sum_{\nodeidx,\nodeidx' \in \mathcal{C}} 
	\edgeweight_{\nodeidx,\nodeidx'}  \underbrace{\normgeneric{\widehat{\weights}^{(\nodeidx)} \!-\!\widehat{\weights}^{(\nodeidx')}}{2}^{2}}_{\stackrel{\eqref{equ_def_error_component_clustered}}{=}\normgeneric{\widetilde{\weights}^{(\nodeidx)} \!-\!\widetilde{\weights}^{(\nodeidx')}}{2}^{2}}  \nonumber \\ 
	& \stackrel{\eqref{equ_lower_bound_tv_eigval}}{\geq} \regparam \eigval{2}\big( \mL^{(\mathcal{C})} \big) \sum_{\nodeidx\in \mathcal{C}}\normgeneric{\widetilde{\weights}^{(\nodeidx)}}{2}^{2}.
\end{align}
If the bound \eqref{equ_upper_bound_err_component_cluster} would not hold, then 
by \eqref{equ_obj_function_gtvmin_average_ac_component_clustered} and \eqref{equ_norm_error_func_value_clustered} 
we would obtain $f' \big( \widehat{\weights}^{(\nodeidx)} \big) > f' \big( \overline{\weights}^{(\nodeidx)} \big)$, which 
contradicts the fact that $\widehat{\weights}^{(\nodeidx)}$ solves \eqref{equ_def_gtvmin}.  

\section{Acknowledgement} 
The authors is grateful for funding received from the Research Council of Finland (decision nr.\ 331197, 331197) 
and European Union (grant nr.\ 952410). Feedback received from Pedro Nardelli and Xu Yang is acknowledged 
warmly. 

\bibliographystyle{IEEEtran}
\bibliography{../Literature.bib}

\end{document}

%% file: ml_macros.tex
\newcommand\localmodel[1]{\hypospace^{(#1)}}

\newcommand\defeq{:=}

\newcommand{\vx}[0]{{\bf x}}
\newcommand{\vv}[0]{{\bf v}}
\newcommand{\vu}[0]{{\bf u}}

\newcommand{\mX}[0]{{\bf X}}

\newcommand{\mL}[0]{{\bf L}}

\newcommand{\vw}[0]{{\bf w}}

\newcommand{\mP}{\mathbf{P}}
\newcommand{\mI}{\mathbf{I}}

\newcommand{\vy}[0]{{\bf y}}

\newcommand{\vz}[0]{{\bf z}}

\newcommand{\vc}[0]{{\bf c}}

\newcommand{\normgeneric}[2]{\left\Vert  {#1} \right\Vert_{#2}}

\newcommand{\bmx}[0]{\begin{bmatrix}}
\newcommand{\emx}[0]{\end{bmatrix}}

\newcommand\truelabel{y}

\newcommand\labelvec{\vy}
\newcommand\featurevec{\vx}

\newcommand{\hypospace}{\mathcal{H}}

\newcommand{\eigval}[1]{\lambda_{#1}}

\newcommand{\regparam}{\alpha}

\DeclareMathOperator*{\argmin}{argmin}

\newcommand{\cluster}{\mathcal{C}}

\newcommand{\featuremtx}{\mX}

\newcommand{\weights}{\vw}

\newcommand{\bd[1]}{\left| \partial #1\right|}

\newcommand{\locallossfunc}[2]{L_{#1}\left(#2 \right)}

\newcommand{\localdataset}[1]{\mathcal{D}^{(#1)}}
\newcommand{\edges}{\mathcal{E}}

\newcommand{\edgeweight}{A}

\newcommand{\graph}{\mathcal{G}}
\newcommand{\nodes}{\mathcal{V}}
\newcommand{\LapMat}[1]{\mL^{(#1)}}

\newcommand{\nodeidx}{i}
\newcommand{\nrnodes}{n}

\newcommand{\edge}[2]{\{#1,#2\}}

\newcommand{\localsamplesize}[1]{m_{#1}}

\newcommand{\dimlocalmodel}{d}
\newcommand{\pair}[2]{\left( #1,#2 \right)}
\newcommand{\localparams}[1]{\mathbf{w}^{(#1)}}

\newcommand{\estlocalparams}[1]{\widehat{\mathbf{w}}^{(#1)}}
